\documentclass{article}

\usepackage[utf8]{inputenc}
\usepackage[T1]{fontenc}
\usepackage{lmodern}
\usepackage{hyperref}
\usepackage{url}
\usepackage{booktabs}
\usepackage{amsfonts}
\usepackage{nicefrac}
\usepackage{microtype}
\usepackage{graphicx}
\usepackage{geometry}
\usepackage{amsmath}
\usepackage{float}
\usepackage{subcaption}
\usepackage{amsthm}

\usepackage{tikz}
\usepackage{pgfplots}
\pgfplotsset{compat=1.18}
\usetikzlibrary{shapes.geometric, arrows, positioning, calc, decorations.pathreplacing}

\tikzstyle{model} = [rectangle, rounded corners, minimum width=2.8cm, minimum height=1.2cm, text centered, draw=black, fill=blue!10, line width=0.8pt]
\tikzstyle{router} = [diamond, aspect=2, minimum width=4cm, minimum height=1cm, text centered, draw=black, fill=orange!20, line width=1pt]
\tikzstyle{arrow} = [thick,->,>=stealth]
\tikzstyle{exit} = [rectangle, draw=gray, fill=gray!10, minimum width=2.5cm, text centered, dashed]

\newtheorem{theorem}{Theorem}
\newtheorem{definition}{Definition}

\title{\textbf{Pyramid MoA: A Probabilistic Framework for \\ Cost-Optimized Anytime Inference}}
\author{\textbf{Arindam Khaled} \\ Independent Researcher \\ \texttt{arindamkhaled@gmail.com}}
\date{April 12, 2026}

\begin{document}

\maketitle

\begin{abstract}
Large Language Models (LLMs) face a persistent trade-off between inference cost and reasoning capability. While larger ``Oracle'' models generally achieve higher accuracy, they are 
more expensive for high-volume deployment. Smaller, cost-effective models 
often struggle with complex tasks. The question is not which tier to 
choose, but how to allocate queries across tiers dynamically. We observe that the emerging practice of LLM cascading and routing implicitly solves an \emph{anytime computation} problem---a class of algorithms, well-studied in classical AI, that produce valid solutions immediately and improve them as additional computation is allocated. In this work, we formalize this connection and propose \textbf{``Pyramid MoA''}, a hierarchical Mixture-of-Agents architecture governed by a decision-theoretic router that dynamically escalates queries only when necessary. We establish a \emph{Probabilistic Anytime Property}, proving that expected solution quality is monotonically non-decreasing with computational depth under identifiable conditions on router precision. We derive a generalized escalation rule from Value of Computation theory that accounts for imperfect oracles, extending the classical monitoring framework of Hansen and Zilberstein to stochastic LLM inference. On the MBPP code generation benchmark, the Consensus Router intercepts \textbf{81.6\%} of bugs. On the GSM8K/MMLU mathematical reasoning benchmark, the system nearly matches 
the Oracle baseline of \textbf{68.1\%} accuracy at the near break-even operating 
point, and achieves \textbf{42.9\%} compute savings in economy mode with a modest 
accuracy trade-off (63.2\%).  Crucially, the router transfers zero-shot to unseen benchmarks: on HumanEval 
it matches the Oracle baseline of \textbf{81.1\%} accuracy, and achieves 
\textbf{62.7\%} cost savings in economy mode (73.2\% accuracy). On the highly 
complex MATH~500 benchmark, the system preserves the \textbf{58.0\%} Oracle 
ceiling and enables up to \textbf{59.0\%} savings in efficiency mode (40.2\% 
accuracy). We further investigate context-aware escalation---where the Oracle receives 
Layer~1 outputs as context---and discover a \emph{context-conditioned anchoring 
effect} consistent across four benchmarks: passing correct SLM reasoning 
improves Oracle accuracy by up to \textbf{19.2 percentage points}, while 
passing incorrect reasoning degrades it by up to \textbf{18.0 percentage 
points}, revealing a fundamental tension in hierarchical Mixture-of-Agents 
architectures. The framework adapts its behavior to task difficulty: enabling substantial 
cost savings on benchmarks where SLMs are strong (up to 42.9\% on 
GSM8K/MMLU) while preserving Oracle-level accuracy on difficult 
out-of-distribution benchmarks (MATH~500, HumanEval).
\end{abstract}

\section{Introduction}

The rapid proliferation of Large Language Models (LLMs) has created a 
persistent tension between inference cost and reasoning capability. 
Recent Mixture-of-Agents (MoA) approaches \cite{wang2024mixture} 
demonstrate that layering multiple models---leveraging the principle 
that an ensemble of weaker learners can outperform any individual 
model---produces stronger outputs than single-model inference. However, 
standard MoA executes all layers for every query regardless of 
difficulty, ignoring the heterogeneous computational needs of different 
inputs. While emerging works like \textit{Sparse MoA} \cite{smoa2024} and \textit{Residual MoA} (RMoA) \cite{rmoa2025} have begun to explore adaptive termination, they often require architectural modifications or complex internal metrics.

We observe that this challenge---deciding how much computation to allocate to a given query---is an instance of a well-studied problem in classical AI: \textbf{anytime computation} \cite{dean1988analysis, zilberstein1996using}. An anytime algorithm produces a valid solution immediately and monotonically improves it as additional computation is allocated. The associated \emph{monitoring problem} \cite{hansen2001monitoring} asks: when should we stop allocating computation and return the current solution? The LLM routing community has been independently reinventing these concepts---early exit strategies, cascade thresholds, confidence-based routing \cite{chen2023frugalgpt, ong2024routellm, teerapittayanon2016branchynet}---without the formal toolkit to analyze them.

We propose \textbf{Pyramid MoA}, a framework that explicitly bridges anytime computation theory and multi-model LLM inference. The name reflects the system's routing geometry: all queries enter at the broad base, where they are processed by a cost-effective ensemble of SLMs. A lightweight router then filters this workload, escalating only a narrowing subset of increasingly difficult queries to the expensive Oracle at the apex. This pyramid-shaped workload distribution---high query volume at the base, narrow volume at the top---ensures that heavy computational investment is concentrated exclusively on the tasks that require it most. By recasting the routing decision as an instance of the classical monitoring problem, we obtain formal guarantees and principled design criteria that ad-hoc cascading approaches lack. Our contributions are:

\textbf{1. Anytime Inference Framework:} We formalize multi-model LLM routing as a probabilistic anytime computation problem. Classical anytime algorithms guarantee per-instance monotonic improvement; we establish a \emph{Probabilistic Anytime Property} (Theorem~\ref{thm:monotonicity}) showing that expected solution quality is monotonically non-decreasing with computational depth under identifiable conditions on router precision. We introduce \emph{performance profiles} adapted from the anytime literature that characterize the cost-quality trade-off and diagnose dataset difficulty.

\textbf{2. Generalized Decision-Theoretic Router:} We derive an optimal escalation rule from Value of Computation theory, generalizing the classical monitoring framework to handle stochastic, imperfect oracles. Unlike prior formulations that assume near-perfect Oracle accuracy, our generalized decision rule (Equation~\ref{eq:generalized_rule}) explicitly accounts for Oracle error, revealing two distinct barriers to escalation: the cost ratio and Oracle imperfection. The resulting router is lightweight, model-agnostic, and compatible with black-box APIs.

\textbf{3. Empirical Dynamic Range \& Context-Aware Analysis:} We demonstrate that the framework adapts its behavior to task 
difficulty---enabling substantial cost savings on benchmarks where SLMs 
are strong while preserving Oracle-level accuracy on difficult 
out-of-distribution benchmarks---and transfers zero-shot to unseen 
benchmarks (HumanEval, MATH 500), validating the theoretical predictions 
across four diverse evaluations. We further reveal a context-conditioned anchoring effect in hierarchical MoA, consistent across four benchmarks and two task domains: passing incorrect SLM reasoning to the Oracle degrades its accuracy by 14.9 to 18.0 percentage points, validating the routing-based design as a necessary safeguard.

\section{Methodology}

\subsection{From Anytime Search to Anytime Inference}
In classical AI, an \emph{anytime algorithm} \cite{dean1988analysis, zilberstein1996using} produces a valid solution immediately and monotonically improves it as additional computation is allocated. This property is formalized through a \emph{performance profile}---a mapping from computation time to solution quality---and a \emph{monitoring problem}---determining when the marginal gain from further computation no longer justifies its cost \cite{hansen2001monitoring}.

We observe that the emerging practice of LLM cascading and routing (e.g., FrugalGPT \cite{chen2023frugalgpt}, RouteLLM \cite{ong2024routellm}) implicitly solves an anytime computation problem. A small model produces an initial answer (valid but potentially suboptimal), and a routing decision determines whether to allocate additional computation via a larger model. However, existing approaches lack the formal framework to analyze this trade-off rigorously. They typically rely on ad-hoc confidence thresholds without characterizing the conditions under which escalation provably improves outcomes.

We bridge this gap by recasting multi-model LLM inference as a \textbf{probabilistic anytime} computation problem. The key departure from the classical setting is that LLM inference is inherently stochastic: unlike deterministic search, where more computation guarantees monotonic improvement, a larger model may occasionally produce a worse answer than a smaller one on any individual query. Our framework addresses this by establishing monotonicity guarantees \emph{in expectation} over query distributions rather than per-instance.

To make the connection precise, we establish the following correspondence:

\begin{table}[H]
\centering
\caption{Mapping from classical anytime computation to Pyramid MoA.}
\label{tab:anytime_mapping}
\begin{tabular}{@{}ll@{}}
\toprule
\textbf{Anytime Computation} & \textbf{Pyramid MoA} \\ \midrule
Initial solution & Layer 1 (SLM ensemble) output \\
Extended computation & Escalation to higher-capability layers \\
Computation time axis & Model capacity / inference cost \\
Performance profile & Accuracy as a function of cost \\
Monitoring problem & Router's escalation decision at each layer \\
Quality guarantee (deterministic) & Quality guarantee (probabilistic) \\
\bottomrule
\end{tabular}
\end{table}

This mapping is not merely analogical. In Section~\ref{sec:probabilistic_anytime}, we show that under identifiable conditions on the router's precision, Pyramid MoA satisfies a formal probabilistic analogue of the classical anytime monotonicity property. In Section~\ref{sec:generalized_routing}, we derive the optimal monitoring policy from the decision-theoretic framework, generalizing the classical Value of Computation to handle stochastic, imperfect oracles.

\subsection{System Components}
\label{sec:system_components}
Our architecture consists of three components:
\begin{enumerate}
    \item \textbf{Layer 1 (The Crowd):} An ensemble of cost-effective models: Llama-3.1-8B-Instruct, Qwen2.5-7B-Instruct, and Gemma-2-9B.
    \item \textbf{The Router:} A trained classifier predicting the probability of Layer 1 failure ($P_{fail}$).
    \item \textbf{Layer 2 (The Oracle):} Llama-3.3-70B-Instruct, invoked only when $P_{fail} > t$, where $t$ is a tunable threshold.
\end{enumerate}

\textbf{Note on Nomenclature (Routing vs.\ Context-Aware):} Unlike context-aware MoA approaches that synthesize a new response from peer outputs \cite{wang2024mixture}, our framework employs a \emph{Routing-Based MoA}. The ensemble's collective signal is used to calibrate the escalation decision rather than to generate the output token sequence. The system components are strictly modular and API-compatible: any black-box model can serve as Layer 1 or Layer 2 without architectural modification, differentiating our approach from methods like RMoA \cite{rmoa2025} that require internal model access. We discuss extensions to context-aware Pyramid MoA---where the Oracle receives Layer 1 outputs as context for refinement---in Section~\ref{sec:discussion}.

\begin{figure}[H]
    \centering
    \begin{tikzpicture}[node distance=2.5cm]
    \node (m2) [model] {\textbf{Qwen2.5-7B}};
    \node (m1) [model, left of=m2, xshift=-1.5cm] {\textbf{Llama-3.1-8B}};
    \node (m3) [model, right of=m2, xshift=1.5cm] {\textbf{Gemma-2-9B}};
    \draw [dashed, gray, line width=0.5pt] ($(m1.north west)+(-0.3,0.3)$)-| ($(m3.south east)+(0.3,-0.3)$)-| ($(m1.north west)+(-0.3,0.3)$);
    \node[anchor=north] at (m2.south) [yshift=-0.5cm] {\small \textbf{Layer 1: Ensemble of SLMs (The Crowd)}};
    \node (router) [router, above of=m2, yshift=2.5cm] {\textbf{Anytime Router}};
    \draw [arrow] (m1.north) .. controls +(up:1.2cm) and +(left:1.2cm) .. (router.west);
    \draw [arrow] (m2.north) -- (router.south);
    \draw [arrow] (m3.north) .. controls +(up:1.2cm) and +(right:1.2cm) .. (router.east);
    \node (features) [rectangle, fill=white, draw=gray, rounded corners, align=center, yshift=1.2cm] at (m2.north) {\scriptsize \textbf{Ensemble Features:} \\ \scriptsize Semantic Agreement \& Intrinsic Logprobs};
    \node (oracle) [model, above of=router, yshift=2cm, fill=green!15] {\textbf{Layer 2: Oracle (Llama-3.3-70B)}};
    \node (exit) [exit, left of=oracle, xshift=-4cm] {\textbf{Output L1 Result}};
    \draw [arrow] (router.north) -- node[anchor=west] {\textbf{Escalate} ($P_{fail} > t$)} (oracle.south);
    \draw [arrow] (router.west) -| node[anchor=south, xshift=1.5cm] {\textbf{Short-Circuit} ($P_{fail} \le t$)} (exit.south);
    \end{tikzpicture}
    \caption{\textbf{Pyramid MoA Architecture:} The system extracts ensemble-wide features from the Layer 1 models to estimate $P_{fail}$. The router solves the anytime monitoring problem---deciding whether to allocate additional computation via the Oracle.}
    \label{fig:arch_diagram}
\end{figure}

\subsection{The Probabilistic Anytime Property}
\label{sec:probabilistic_anytime}
We now formalize the sense in which Pyramid MoA exhibits anytime behavior.

\begin{definition}[Deterministic Anytime Property]
An algorithm $A$ satisfies the \emph{anytime property} if, for any input $x$ and computation budgets $c_1 < c_2$, the solution quality satisfies $Q(A, x, c_2) \geq Q(A, x, c_1)$.
\end{definition}

This is the classical definition, originating with Dean and Boddy \cite{dean1988analysis} and formalized by Zilberstein \cite{zilberstein1996using}. It requires per-instance monotonicity, which cannot hold in stochastic LLM inference---a 70B model can produce an incorrect answer on a query that an 8B model answers correctly. We therefore relax the requirement:

\begin{definition}[Probabilistic Anytime Property]
A system $S$ satisfies the \emph{probabilistic anytime property} with respect to a query distribution $\mathcal{D}$ if, for computational depths $d_1 < d_2$:
\begin{equation}
    \mathbb{E}_{x \sim \mathcal{D}}\big[Q(S, x, d_2)\big] \geq \mathbb{E}_{x \sim \mathcal{D}}\big[Q(S, x, d_1)\big]
    \label{eq:prob_anytime}
\end{equation}
\end{definition}

\begin{theorem}[Monotonicity Condition]
\label{thm:monotonicity}
Let $R$ denote the set of queries escalated by the router (i.e., those with $P_{fail} > t$), and let $\bar{R}$ denote the queries retained at Layer 1. Define $\alpha_{L1}(R)$ and $\alpha_{L2}(R)$ as the accuracy of Layer 1 and Layer 2 on the escalated subset, and $p_R = |R|/N$ as the escalation rate. The Pyramid MoA system satisfies the probabilistic anytime property if and only if:
\begin{equation}
    \alpha_{L2}(R) \geq \alpha_{L1}(R)
    \label{eq:monotonicity}
\end{equation}
\end{theorem}

\begin{proof}
The accuracy of the full system is $\text{Acc}_{\text{MoA}} = (1 - p_R) \cdot \alpha_{L1}(\bar{R}) + p_R \cdot \alpha_{L2}(R)$, while Layer 1 alone achieves $\text{Acc}_{L1} = (1 - p_R) \cdot \alpha_{L1}(\bar{R}) + p_R \cdot \alpha_{L1}(R)$. The improvement is therefore:
\begin{equation}
    \text{Acc}_{\text{MoA}} - \text{Acc}_{L1} = p_R \cdot \big[\alpha_{L2}(R) - \alpha_{L1}(R)\big]
    \label{eq:improvement}
\end{equation}
Since $p_R > 0$ whenever escalation occurs, this is non-negative if and only if $\alpha_{L2}(R) \geq \alpha_{L1}(R)$.
\end{proof}

\textbf{Remark 1} (Router quality amplifies the anytime property). A perfect router that escalates only queries where L1 fails maximizes $\alpha_{L2}(R) - \alpha_{L1}(R)$. A random router reduces the gap to the marginal difference $\alpha_{L2} - \alpha_{L1}$ across the full distribution. We verify this condition empirically in Section~\ref{sec:monotonicity_verification}.

\subsection{Generalized Decision-Theoretic Routing}
\label{sec:generalized_routing}
We formalize the escalation decision as a decision-theoretic problem, generalizing prior formulations to handle the realistic case where the Oracle is imperfect. Let $P_{fail}$ be the router's estimated failure probability, $P_{oracle}$ the Oracle's success probability, $U_{correct}$ the utility of a correct answer, and $C_{esc}$ the escalation cost. The expected utilities of stopping vs.\ escalating are:
\begin{equation}
    \mathbb{E}[U_{\text{stop}}] = (1 - P_{fail}) \cdot U_{correct}, \quad
    \mathbb{E}[U_{\text{esc}}] = P_{oracle} \cdot U_{correct} - C_{esc}
    \label{eq:utilities}
\end{equation}
Escalation is optimal when $\mathbb{E}[U_{\text{esc}}] > \mathbb{E}[U_{\text{stop}}]$, yielding:
\begin{equation}
    P_{fail} > \underbrace{\frac{C_{esc}}{U_{correct}}}_{\text{cost barrier}} + \underbrace{(1 - P_{oracle})}_{\text{imperfection barrier}}
    \label{eq:generalized_rule}
\end{equation}

This reveals two distinct barriers to escalation. The \textbf{cost barrier} reflects the computational expense, as in standard cascade formulations. The \textbf{imperfection barrier} captures the risk that escalation incurs cost without producing a correct answer. When $P_{oracle} \to 1$, the imperfection barrier vanishes and Equation~\ref{eq:generalized_rule} reduces to the standard cascade rule.

In practice, when the utility of a correct answer substantially exceeds 
the escalation cost ($U_{correct} \gg C_{esc}$), the cost barrier 
becomes negligible and the imperfection barrier dominates the escalation 
decision. This distinguishes our framework from purely cost-driven 
cascades: the primary reason to avoid escalation is not expense but 
Oracle unreliability.

\textbf{Connection to classical monitoring.} Our router solves a single-step version of Hansen and Zilberstein's monitoring problem \cite{hansen2001monitoring}, observing ensemble-level features as a proxy for current solution quality and making a binary stop/escalate decision. This reduction from a sequential to a single-step decision is justified by the two-tier architecture. We discuss extensions to deeper pyramids in Section~\ref{sec:discussion}.

\subsection{Performance Profiles}
\label{sec:performance_profiles}
A \emph{performance profile} \cite{zilberstein1996using, hansen2001monitoring} maps computational investment to expected solution quality. We define Pyramid MoA's performance profile as $\Pi(t) = \mathbb{E}_{x \sim \mathcal{D}}[Q(S, x) \mid \text{threshold} = t]$. At $t = 1.0$ (never escalate), cost is minimal and accuracy equals $\alpha_{L1}$. At $t = 0.0$ (always escalate), cost is maximal and accuracy approaches $\alpha_{L2}$. A well-calibrated router produces a \emph{concave} profile---accuracy rises steeply at first and flattens with additional spending---enabling favorable operating points where most accuracy gain is captured at a fraction of the full cost. We present empirical performance profiles in Section~\ref{sec:experiments}.

\section{Experimental Results}
\label{sec:experiments}

We evaluate Pyramid MoA across four benchmarks spanning code generation and mathematical reasoning. All experiments use a unified Layer 1 ensemble of Llama-3.1-8B-Instruct, Qwen2.5-7B-Instruct, and Gemma-2-9B, with Llama-3.3-70B-Instruct as the Layer 2 Oracle. We train two domain-specific routers and evaluate their in-domain performance and zero-shot transfer.

\begin{table}[h]
\centering
\caption{Router configurations by domain.}
\label{tab:router_comparison}
\begin{tabular}{@{}lll@{}}
\toprule
\textbf{Feature} & \textbf{Consensus Router} & \textbf{Anytime Router} \\ \midrule
\textbf{Algorithm} & XGBoost & XGBoost \\
\textbf{Training Domain} & MBPP (Code) & GSM8K/MMLU (Math) \\
\textbf{Primary Signal} & Semantic Agreement & Token Log-Probabilities \\
\textbf{Routing Characteristic} & Error-catching (high recall) & Selective escalation (high precision) \\
\textbf{Transfer Evaluation} & HumanEval & MATH 500 \\
\bottomrule
\end{tabular}
\end{table}

\subsection{Experiment I: Code Generation (MBPP)}
For code generation tasks, we deployed the \textbf{Consensus Router} (XGBoost), prioritizing safety via semantic agreement among the ensemble's outputs.

On the MBPP holdout set ($N=156$, threshold $t=0.4$), the router achieved a \textbf{Recall of 81.6\%}, intercepting 62 out of 76 erroneous code snippets (Figure~\ref{fig:mbpp_mechanism}a). Feature importance analysis (Figure~\ref{fig:mbpp_mechanism}b) validates the ``Consensus'' hypothesis: \textit{Self-Reported Confidence} was the least predictive feature, while \textit{Output Length Variance} and \textit{Semantic Agreement} were dominant predictors. This confirms that for open-ended generation tasks, models are often ``confident but wrong'' \cite{guo2017calibration}, necessitating external peer-review signals.

\begin{figure}[h]
    \centering
    \begin{subfigure}[b]{0.48\textwidth}
        \centering
        \includegraphics[width=\textwidth]{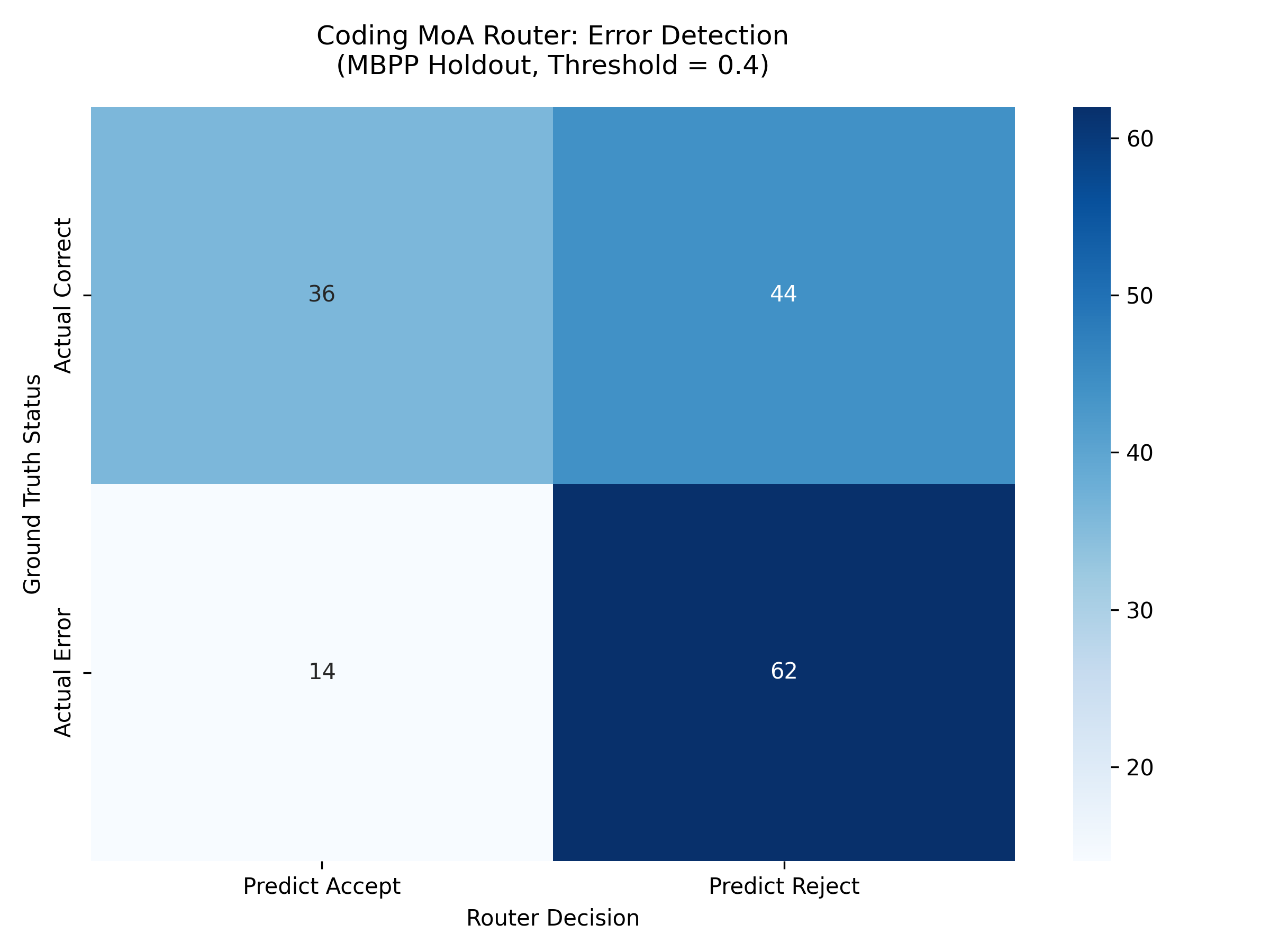}
        \caption{\textbf{Confusion Matrix (MBPP, $t=0.4$)}}
        \label{fig:mbpp_confusion}
    \end{subfigure}
    \hfill
    \begin{subfigure}[b]{0.48\textwidth}
        \centering
        \includegraphics[width=\textwidth]{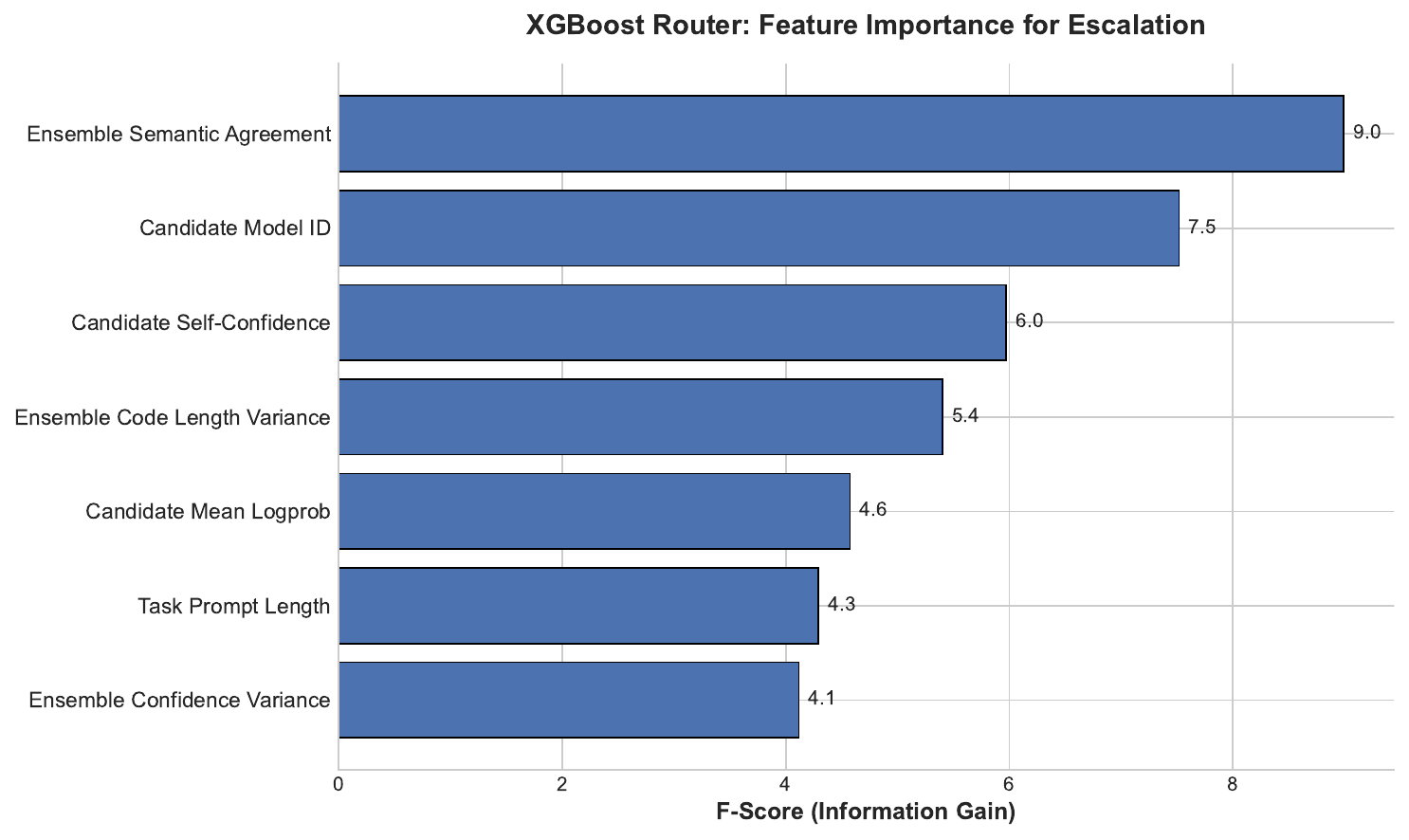}
        \caption{\textbf{Feature Importance (MBPP)}}
        \label{fig:mbpp_importance}
    \end{subfigure}
    \caption{\textbf{Consensus Mechanism:} Evaluation on MBPP showing that peer-agreement signals significantly outperform intrinsic model confidence for error detection.}
    \label{fig:mbpp_mechanism}
\end{figure}

\subsubsection{Zero-Shot Transfer: HumanEval}
To test generalizability, we applied the MBPP-trained Consensus Router zero-shot to the HumanEval benchmark. As shown in Figure~\ref{fig:humaneval_transfer}, the router successfully 
transfers: at the ``Baseline Match'' operating point, the system achieves the 
full Oracle accuracy of \textbf{81.1\%} while still requiring substantially 
fewer Oracle calls than running the 70B model on all queries. In ``Economy 
Mode,'' the system achieves 73.2\% accuracy with \textbf{62.7\%} compute 
savings versus the Oracle.

\begin{figure}[h]
    \centering
    \includegraphics[width=0.95\linewidth]{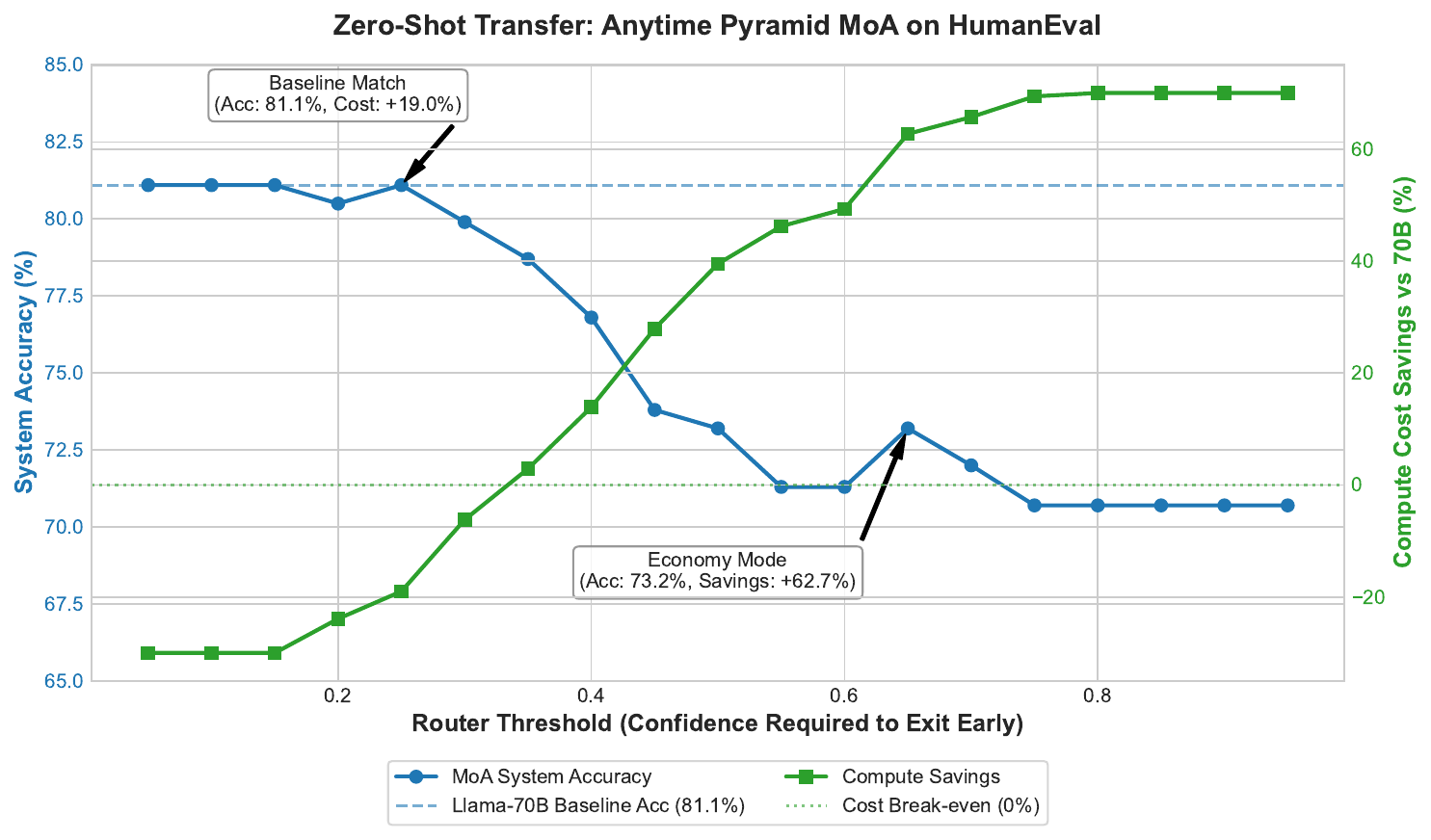}
    \caption{\textbf{Zero-Shot Transfer to HumanEval:} The MBPP-trained Consensus Router transfers effectively, achieving the Oracle baseline (81.1\%) and enabling up to 62.7\% cost savings in economy mode.}
    \label{fig:humaneval_transfer}
\end{figure}

\subsection{Experiment II: Mathematical Reasoning (GSM8K/MMLU)}
For convergent mathematical tasks, we deployed the \textbf{Anytime Router} (XGBoost), utilizing intrinsic token log-probabilities ($avg\_logprob$, $min\_prob$) as the primary routing signal.

\subsubsection{Router Analysis}
On the GSM8K/MMLU holdout set ($N=1{,}053$, threshold $t=0.4$), the router achieved a \textbf{Recall of 78.5\%} (204 of 260 errors detected); Figure~\ref{fig:math_mechanism}a). Feature importance analysis (Figure~\ref{fig:math_mechanism}b) reveals a markedly different signal structure from code generation: \textit{Candidate: Llama-3.1-8B} (the primary model's correctness indicator) and \textit{Confidence Score} are the dominant features, with \textit{Minimum Token Probability} ($min\_prob$) providing complementary signal. Unlike MBPP, intrinsic confidence is highly predictive for convergent tasks where answers are deterministic.

\begin{figure}[h]
    \centering
    \begin{subfigure}[b]{0.48\textwidth}
        \centering
        \includegraphics[width=\textwidth]{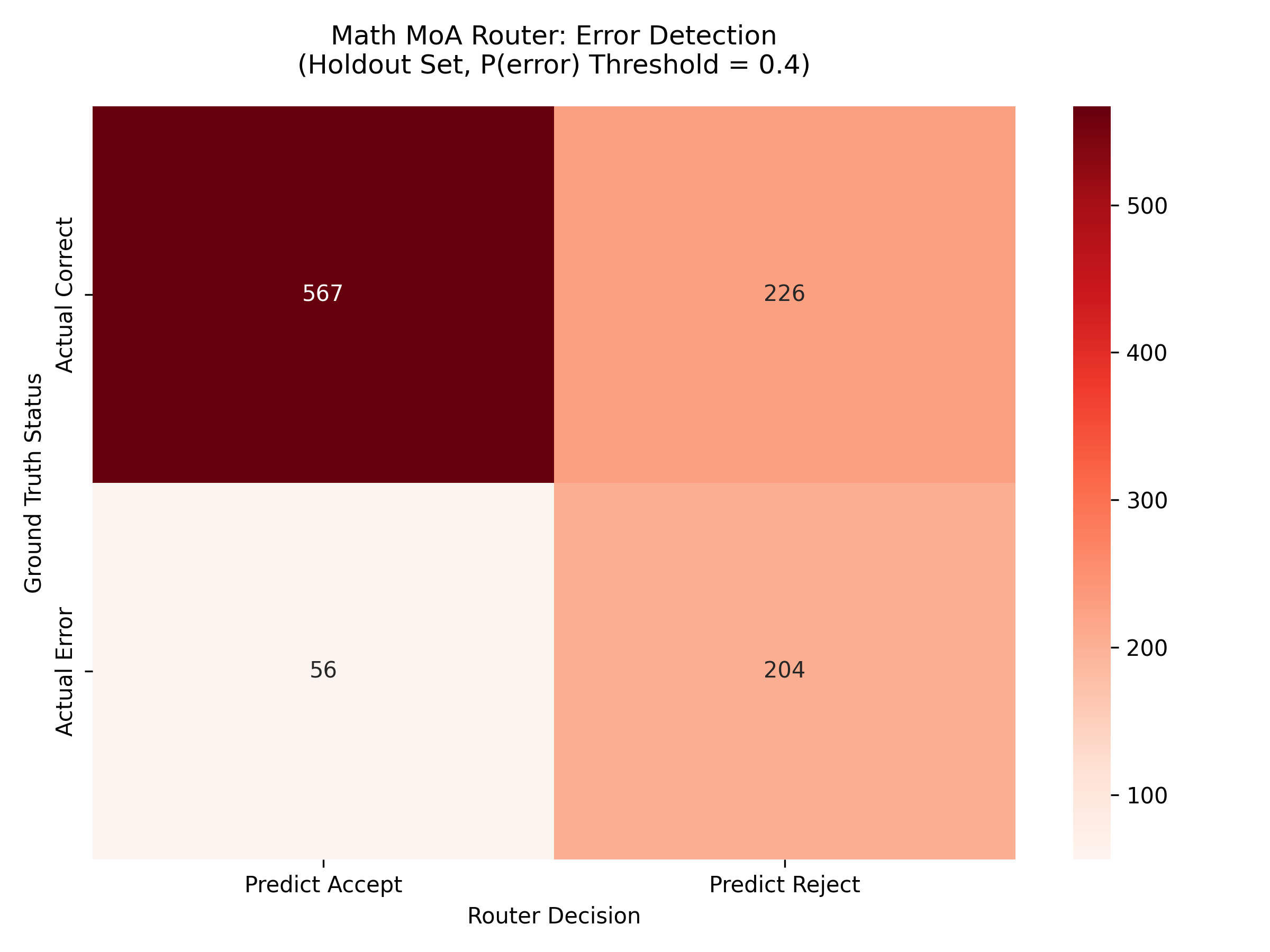}
        \caption{\textbf{Confusion Matrix (GSM8K/MMLU, $t=0.4$)}}
        \label{fig:math_confusion}
    \end{subfigure}
    \hfill
    \begin{subfigure}[b]{0.48\textwidth}
        \centering
        \includegraphics[width=\textwidth]{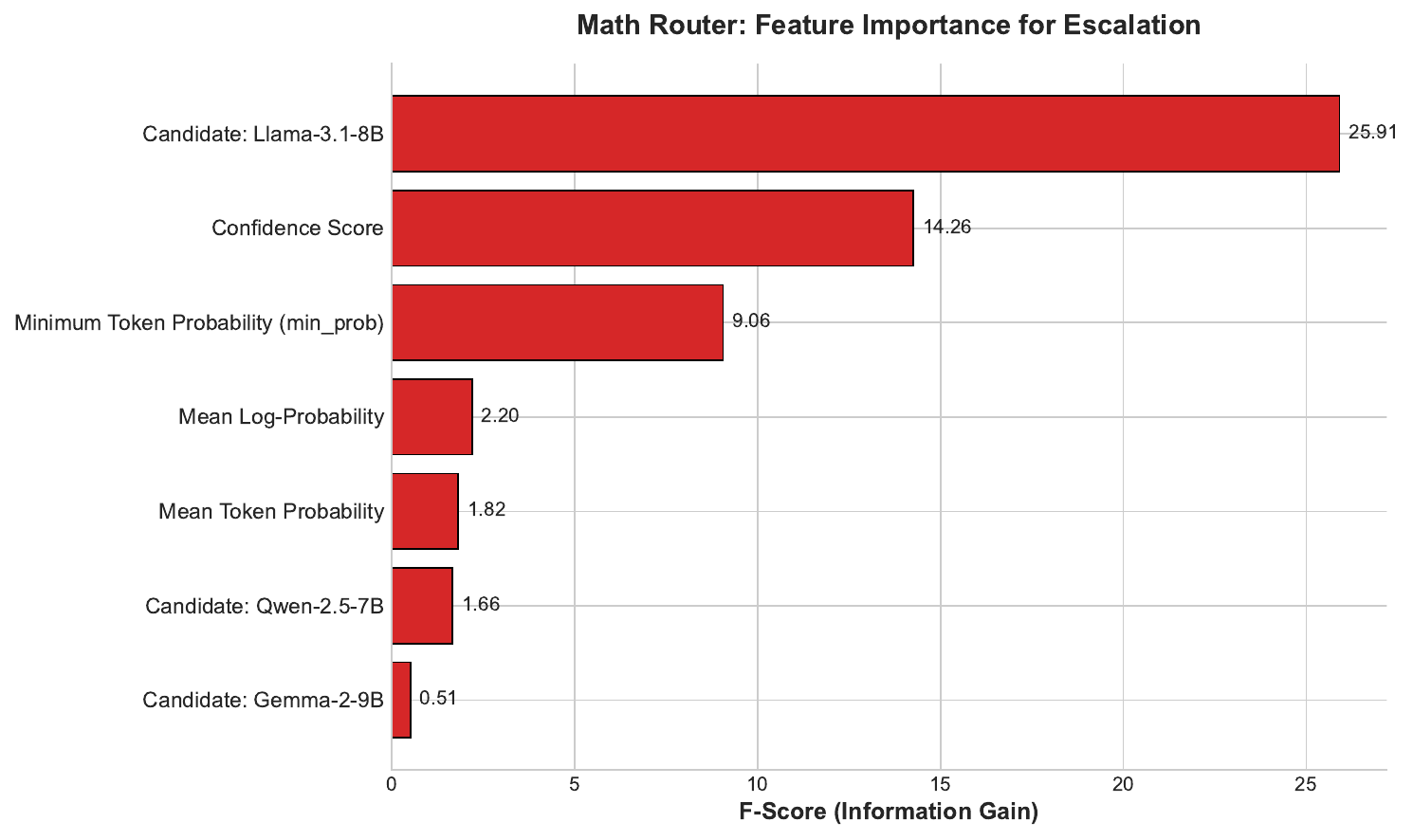}
        \caption{\textbf{Feature Importance (Math Router)}}
        \label{fig:math_importance}
    \end{subfigure}
    \caption{\textbf{Math Router Analysis:} The XGBoost router leverages candidate model correctness and token-level uncertainty signals for escalation decisions on convergent tasks.}
    \label{fig:math_mechanism}
\end{figure}

\subsubsection{The Anytime Performance Profile}
Figure~\ref{fig:math_holdout} presents the full performance profile on the holdout set, sweeping the threshold from $t=0.05$ to $t=0.95$. The Oracle baseline (Llama-3.3-70B) achieves 68.1\% accuracy. At the ``Near Break-even'' operating point ($t \approx 0.15$), the system achieves \textbf{67.8\%} accuracy with \textbf{3.3\%} compute savings---nearly matching the Oracle at reduced cost. In ``Economy Mode'' ($t \approx 0.35$), the system achieves 63.2\% accuracy with \textbf{42.9\%} compute savings. The characteristic concave shape confirms that the router allocates Oracle computation to the highest-value queries first, consistent with the theoretical prediction from Section~\ref{sec:performance_profiles}.

\begin{figure}[h]
    \centering
    \includegraphics[width=0.95\linewidth]{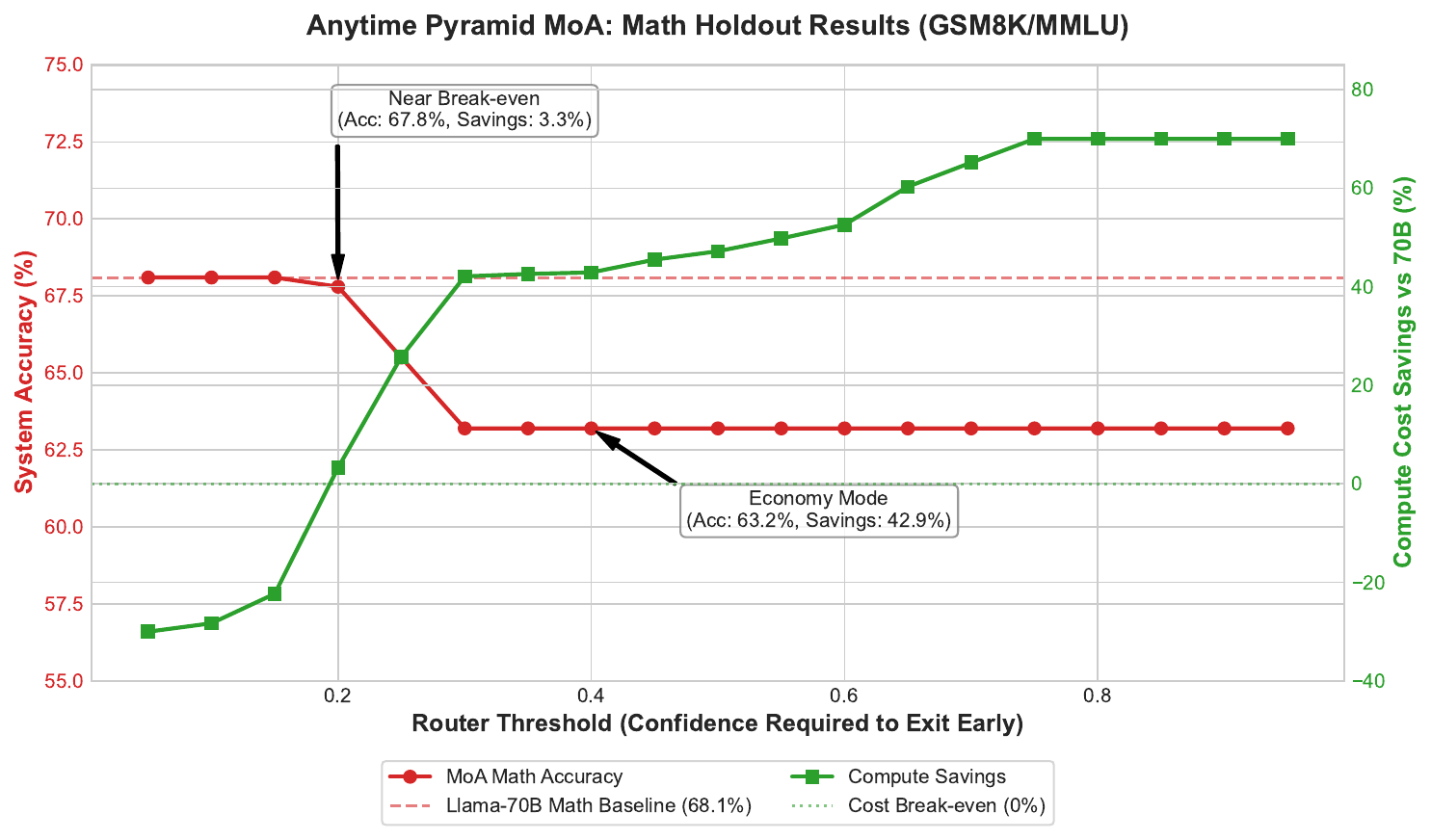}
    \caption{\textbf{Anytime Performance Profile (GSM8K/MMLU Holdout):} The dual-axis plot shows accuracy (red) and compute savings (green) as a function of router threshold. The concave accuracy profile confirms efficient allocation of Oracle computation. The monotonically decreasing accuracy curve as the threshold increases (i.e., as less computation is allocated) empirically demonstrates the probabilistic anytime property: expected solution quality is non-decreasing with computational depth.}
    \label{fig:math_holdout}
\end{figure}
Notably, a slight accuracy recovery is visible at higher thresholds ($t > 0.3$), where eliminating the remaining escalations marginally \emph{improves} system accuracy. This reflects the imperfection barrier (Equation~\ref{eq:generalized_rule}) in practice: at these operating points, a small number of escalations replace correct Layer~1 answers with incorrect Oracle answers. Removing those harmful escalations yields a net accuracy gain, providing direct empirical evidence that Oracle imperfection is not merely a theoretical concern but a measurable effect in the performance profile.

\subsubsection{Zero-Shot Transfer: MATH 500}
To test robustness against severe domain shifts, we evaluated the GSM8K/MMLU-trained router zero-shot on the MATH 500 dataset, which contains AIME-level calculus and algebra problems well outside the training distribution.

As shown in Figure~\ref{fig:math500_transfer}, the router successfully detects the distributional shift. At the ``Zero-Shot Baseline Match'' point ($t \leq 0.25$), the system achieves the full Oracle accuracy of \textbf{58.0\%}. In ``Efficiency Mode,'' the system achieves 40.2\% accuracy with \textbf{59.0\% compute savings}. This behavior is \emph{predicted} by the generalized decision rule 
(Equation~\ref{eq:generalized_rule}): Layer~1 accuracy is low on this 
difficult distribution, producing high $P_{fail}$ estimates that trigger 
frequent escalation. At the same time, with $P_{oracle} \approx 0.58$, the 
imperfection barrier is substantial ($\approx 0.42$), meaning that even 
escalated queries face significant Oracle error risk. The router correctly 
identifies that heavy escalation is needed to preserve accuracy, while the 
imperfection barrier limits the gains achievable through escalation alone.

\begin{figure}[h]
    \centering
    \includegraphics[width=0.95\linewidth]{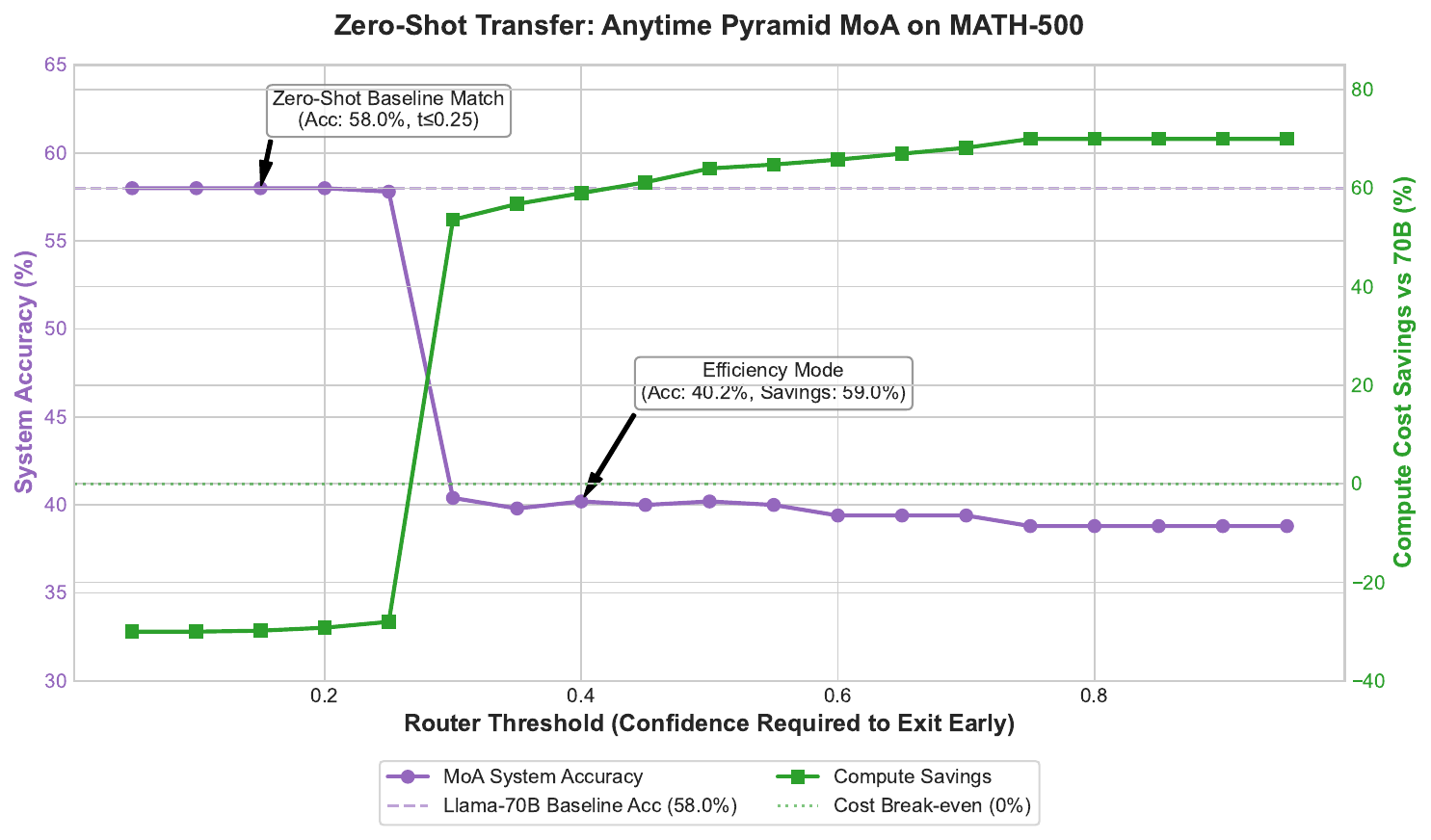}
    \caption{\textbf{Zero-Shot Transfer to MATH 500:} The GSM8K/MMLU-trained router transfers to out-of-distribution problems, preserving the Oracle ceiling (58.0\%) and enabling efficiency gains at higher thresholds.}
    \label{fig:math500_transfer}
\end{figure}

\subsection{Empirical Verification of the Monotonicity Condition}
\label{sec:monotonicity_verification}
Theorem~\ref{thm:monotonicity} establishes that the probabilistic anytime property holds if and only if the Oracle outperforms Layer 1 on the router-escalated subset. Table~\ref{tab:monotonicity} verifies this condition directly across all four benchmarks at a consistent threshold of $t=0.4$.

\begin{table}[h]
\centering
\caption{Verification of the monotonicity condition (Theorem~\ref{thm:monotonicity}) at threshold $t=0.4$. $\alpha_{L1}(R)$ is computed using majority vote over the Layer 1 ensemble. In all cases, $\alpha_{L2}(R) > \alpha_{L1}(R)$, confirming strict probabilistic anytime improvement.}
\label{tab:monotonicity}
\begin{tabular}{@{}lcccccc@{}}
\toprule
\textbf{Benchmark} & \textbf{$p_R$} & \textbf{$\alpha_{L1}(R)$} & \textbf{$\alpha_{L2}(R)$} & \textbf{Gap} & \textbf{$\Delta$Acc} & \textbf{Satisfied?} \\ \midrule
MBPP (Code) & 0.558 & 0.379 & 0.690 & +0.311 & +17.3\% & \checkmark \\
HumanEval (Code, OOD) & 0.561 & 0.696 & 0.804 & +0.109 & +6.1\% & \checkmark \\
GSM8K/MMLU (Math) & 0.271 & 0.632 & 0.768 & +0.137 & +3.7\% & \checkmark \\
MATH 500 (Math, OOD) & 0.110 & 0.145 & 0.582 & +0.436 & +4.8\% & \checkmark \\
\bottomrule
\end{tabular}
\end{table}

The table reveals how router quality and task difficulty interact. On MBPP, the router achieves the largest gap (+31.1pp) by precisely identifying queries where the ensemble disagrees---queries on which Layer 1 accuracy drops to 37.9\% while the Oracle achieves 69.0\%. On HumanEval, the MBPP-trained router transfers zero-shot with a +10.9pp gap, demonstrating that the learned consensus signal generalizes across code generation tasks. On GSM8K/MMLU, the corrected router escalates 27.1\% of queries---targeting the subset where Layer 1 accuracy is 63.2\%, well below the population average---and the Oracle achieves 76.8\% on this harder subset, yielding a +13.7pp gap. On MATH 500, the router escalates only 11.0\% of queries, selecting the most difficult problems where Layer 1 accuracy is just 14.5\%. The Oracle achieves 58.2\% on this subset---a +43.6pp gap---demonstrating that the framework concentrates expensive computation precisely where it yields the largest marginal improvement.

\subsection{Context-Aware Escalation: Anchoring in Hierarchical MoA}
\label{sec:context_aware}

The current framework employs \emph{routing-based} mixture: the Oracle generates its answer independently without access to Layer~1 outputs. A natural question is whether passing the SLM ensemble's chain-of-thought reasoning and candidate answers as context to the Oracle improves accuracy on escalated queries---transforming the architecture from routing-based to \emph{context-aware} Pyramid MoA.

We evaluate this across four benchmarks by providing the Oracle with all available Layer~1 outputs (chain-of-thought and final answers from up to three SLMs) alongside the original query, with instructions to review, reconcile, and correct the candidates. Table~\ref{tab:context_aware} presents the results segmented by whether the SLM majority answer was correct.

\begin{table}[h]
\centering
\caption{Context-aware escalation across four benchmarks. ``Baseline'' denotes 
the Oracle solving independently; ``With Context'' denotes the Oracle receiving 
Layer~1 outputs. Results are segmented by whether the SLM majority answer was 
correct. The anchoring effect---context helps when SLMs are correct, hurts when 
SLMs are wrong---is consistent across both mathematical reasoning and code 
generation tasks.}
\label{tab:context_aware}
\begin{tabular}{@{}llcccc@{}}
\toprule
\textbf{Benchmark} & \textbf{SLM Majority} & \textbf{$N$} & \textbf{Baseline} & \textbf{With Context} & \textbf{$\Delta$} \\ \midrule
MATH 500 & Correct & 239 & 70.3\% & 89.5\% & +19.2 pp \\
         & Wrong   & 261 & 59.0\% & 41.0\% & $-$18.0 pp \\
\midrule
MMLU     & Correct & 301 & 89.4\% & 96.0\% & +6.6 pp \\
         & Wrong   & 251 & 59.8\% & 43.8\% & $-$16.1 pp \\
\midrule
GSM8K    & Correct & 1079 & 99.0\% & 99.4\% & +0.5 pp \\
         & Wrong   & 121 & 88.4\% & 73.6\% & $-$14.9 pp \\
\midrule
HumanEval & Correct & 112 & 89.3\% & 97.3\% & +8.0 pp \\
           & Wrong  & 52  & 63.5\% & 63.5\% & +0.0 pp \\
\bottomrule
\end{tabular}
\end{table}

\textbf{The anchoring effect is general.} The context-conditioned anchoring 
pattern is consistent across all four benchmarks and both task domains. When the 
SLM majority is correct, context improves Oracle accuracy by +0.5 to +19.2pp. 
When the SLM majority is wrong, context degrades Oracle accuracy by $-$14.9 to 
$-$18.0pp across three of four benchmarks. The exception is HumanEval, where 
context has no effect on the majority-wrong subset---possibly because code 
generation tasks provide more objective verification signals (test execution) 
that resist anchoring. GSM8K shows a ceiling effect on the positive side 
(+0.5pp) because the Oracle already achieves 99.0\% on the majority-correct 
subset, leaving little room for improvement. The negative side remains 
substantial ($-$14.9pp), confirming that even on easy benchmarks, incorrect SLM 
reasoning reliably degrades Oracle performance.

\textbf{Architectural tension on the escalated subset.} We examine this tension in detail on MATH~500, where the effect is most pronounced. On the 55 queries escalated by the router at $t=0.4$, 95\% have incorrect SLM majorities---precisely the subset where context is most harmful. On this escalated subset, passing SLM context degrades Oracle accuracy from 58.2\% to 36.4\% ($-$21.8pp). The router escalates records \emph{because} it believes the SLMs are unreliable, yet the context-aware MoA design passes those unreliable outputs as context to the Oracle. This is the subset where context hurts the most.

\textbf{Implication.} The negative result validates the routing-based design: the router is doing real work by identifying genuinely hard problems (Layer~1 accuracy of 14.5\% on the escalated subset), and blindly passing SLM context is actively harmful on exactly those problems. This motivates a natural extension---a \emph{ternary} routing decision (escalate with context, escalate without context, or retain at Layer~1)---discussed in Section~\ref{sec:discussion}.

\section{Related Work}

\textbf{Anytime Algorithms \& the Monitoring Problem:}
The concept of trading computation time for solution quality originates 
with Dean and Boddy \cite{dean1988analysis}, who coined the term 
``anytime algorithm,'' and Horvitz \cite{horvitz1987reasoning}, who 
independently developed the equivalent notion of ``flexible 
computation.'' Zilberstein \cite{zilberstein1996using} extended this 
line to the composition of complex systems from anytime components, and 
Hansen and Zilberstein \cite{hansen2001monitoring} framed the 
\emph{monitoring problem} as a sequential decision process solvable by 
dynamic programming. Our work extends this line to generative AI, casting the multi-model 
routing decision as a single-step monitoring problem in the current 
two-layer architecture. The ``early exit'' strategy in deep learning, exemplified by 
BranchyNet \cite{teerapittayanon2016branchynet}, represents a related 
but architecturally distinct approach: early exit modifies the model's 
internal structure, whereas our framework operates on black-box model 
outputs.

\textbf{Mixture-of-Agents (MoA):}
Collaborative layers of LLMs can outperform individual state-of-the-art models \cite{wang2024mixture}. However, standard MoA implementations execute a pre-determined number of layers for all inputs, ignoring the heterogeneous difficulty of user queries. \textbf{Self-MoA} \cite{selfmoa2025} demonstrates that iteratively querying a single strong model can sometimes outperform diverse ensembles. \textbf{Sparse MoA (SMoA)} \cite{smoa2024} and \textbf{Residual MoA (RMoA)} \cite{rmoa2025} introduce early stopping and residual connections to reduce compute. Unlike RMoA, which relies on architectural residuals, our \textbf{Pyramid MoA} employs a distinct \emph{decision-theoretic} approach grounded in anytime computation theory, providing formal monotonicity guarantees (Theorem~\ref{thm:monotonicity}) and a generalized escalation rule (Equation~\ref{eq:generalized_rule}) that accounts for Oracle imperfection.

\textbf{LLM Cascading \& Routing:}
FrugalGPT \cite{chen2023frugalgpt} introduced cascading as a cost-reduction strategy. RouteLLM \cite{ong2024routellm} trains preference-based routers using chatbot arena data to select between strong and weak models. Our framework subsumes simple cascading as a special case and provides the theoretical apparatus---performance profiles, the monotonicity condition, the generalized decision rule---to analyze when and why cascading works. The key distinction is that our router leverages ensemble agreement across multiple models rather than single-model confidence or pairwise preference data.

\textbf{Connection to Retrieval-Augmented Generation.} This finding parallels 
well-known challenges in retrieval-augmented generation (RAG), where passing 
incorrect retrieved documents to an LLM can cause anchoring---the model commits 
to plausible-but-wrong reasoning derived from the context rather than solving 
independently. In our setting, the Layer~1 ensemble outputs play the role of 
``retrieved candidates.'' On MATH~500, when those candidates are correct, passing them to the 
Oracle improves its accuracy by +19.2pp; when they are incorrect, context 
degrades Oracle performance by $-$18.0pp (59.0\% $\to$ 41.0\%). This pattern 
holds across all four benchmarks (Table~\ref{tab:context_aware}), with the 
negative effect ranging from $-$14.9pp to $-$18.0pp. This suggests that effective context-aware Pyramid MoA requires not 
just passing Layer~1 outputs as context, but \emph{filtering or weighting} those 
outputs based on confidence. A natural extension is to escalate with context only 
when the router's confidence in Layer~1 correctness exceeds a secondary 
threshold---effectively solving a two-threshold routing problem where the system 
decides not just \emph{whether} to escalate, but \emph{how} (with or without 
context). We defer this extension to future work but note that it directly 
motivates the ternary routing strategy outlined in Section~\ref{sec:discussion}.

\section{Discussion \& Conclusion}
\label{sec:discussion}

We presented Pyramid MoA, a framework that bridges classical anytime computation theory and modern multi-model LLM inference. By formalizing the routing decision as an instance of the anytime monitoring problem, we established a Probabilistic Anytime Property (Theorem~\ref{thm:monotonicity}) with provable monotonicity guarantees and derived a generalized escalation rule (Equation~\ref{eq:generalized_rule}) that accounts for Oracle imperfection. Empirically, the framework demonstrates dynamic range across task difficulty: the Consensus Router achieves 81.6\% bug recall on MBPP and transfers zero-shot to HumanEval with 62.7\% cost savings; the Anytime Router nearly matches the 68.1\% Oracle baseline on GSM8K/MMLU with up to 42.9\% compute savings in economy mode, and transfers to MATH 500, preserving the 58.0\% accuracy ceiling. Direct verification of the monotonicity condition 
(Theorem~\ref{thm:monotonicity}, Table~\ref{tab:monotonicity}) confirms 
that the probabilistic anytime property is strictly satisfied across all 
four benchmarks. Furthermore, context-aware escalation experiments reveal a consistent anchoring effect across all four benchmarks ($-$14.9 to $-$18.0pp when SLM majority is wrong), demonstrating that the routing-based design is not merely a cost optimization but a necessary safeguard against context poisoning on hard queries.

\textbf{Limitations \& Future Work.} The context-aware escalation analysis (Section~\ref{sec:context_aware}) reveals that the Oracle cannot reliably discount incorrect SLM reasoning, even when the capability gap is substantial. A natural mitigation is a \emph{ternary} routing decision: the router assigns queries to one of three actions---retain at Layer~1 (high SLM confidence), escalate \emph{with} context (medium confidence, where SLMs are likely correct and context is beneficial), or escalate \emph{without} context (low confidence, where SLM outputs are unreliable). The context-aware anchoring data suggests that such a ternary system could 
achieve accuracy \emph{exceeding} the Oracle baseline: on medium-confidence 
queries where SLM context is correct, the Oracle with context outperforms the 
Oracle alone by up to +19.2pp (Table~\ref{tab:context_aware}), a boost unavailable to the standalone Oracle. This extends the single-threshold decision rule to two thresholds and requires principled threshold transfer for zero-shot evaluation, which we leave to future work. Additionally, extending the architecture to deeper pyramids (e.g., 8B $\to$ 70B $\to$ 405B) would require the full sequential monitoring formulation from Hansen and Zilberstein \cite{hansen2001monitoring}. A self-evaluating feedback loop---where Oracle corrections are used to LoRA fine-tune the base SLM---could reduce escalation rates over time. Finally, future work will incorporate comprehensive statistical significance testing across all benchmarks.


\end{document}